\documentclass{ecai}  

\usepackage{graphicx}
\usepackage{latexsym}

\usepackage{subcaption} 
\usepackage{booktabs} 
\usepackage{comment} 
\usepackage{amsmath}
\usepackage{amssymb}
\usepackage{mathtools}
\usepackage{amsthm}
\usepackage{color}
\usepackage{algorithm}
\usepackage{algorithmic}
\usepackage{hyperref}
\theoremstyle{plain}
\newtheorem{theorem}{Theorem}[section]

\newtheorem{lemma}[theorem]{Lemma}

\theoremstyle{definition}

\theoremstyle{remark}

\newcommand{\CommentedText}[1]{}

\newcommand{\E}{{\rm I\kern-.3em E}}
\newcommand{\Var}{\mathrm{\textbf{Var}}}

\ecaisubmission      

\newcommand\defeq{\mathrel{\stackrel{\makebox[0pt]{\mbox{\normalfont\scriptsize def}}}{:=}}}

\begin{document}

\begin{frontmatter}

\title{Estimating the Robustness Radius for Randomized Smoothing with 100$\times$ Sample Efficiency}

\author[A,B]{\fnms{Emmanouil}~\snm{Seferis}}
\author[B]{\fnms{Stefanos}~\snm{Kollias}}
\author[C]{\fnms{Chih-Hong}~\snm{Cheng}}

\address[A]{Fraunhofer IKS, Germany}
\address[B]{National Technical University of Athens, Greece}
\address[C]{Chalmers University of Technology, Sweden}

\begin{abstract} Randomized smoothing (RS) has successfully been used to improve the robustness of predictions for deep neural networks (DNNs) by adding random noise to create multiple variations of an input, followed by deciding the consensus. To understand if an RS-enabled DNN is effective in the sampled input domains, it is mandatory to sample data points within the operational design domain, acquire the point-wise certificate regarding robustness radius, and compare it with pre-defined acceptance criteria. Consequently, ensuring that a point-wise robustness certificate for any given data point is obtained relatively cost-effectively is crucial. This work demonstrates that reducing the number of samples by one or two orders of magnitude can still enable the computation of a slightly smaller robustness radius (commonly $\approx20\%$ radius reduction) with the same confidence. We provide the mathematical foundation for explaining the phenomenon while experimentally showing promising results on the standard CIFAR-10 and ImageNet datasets.

\end{abstract}

\end{frontmatter}

\section{Introduction}

Deep Neural Networks (DNNs) have achieved impressive results in tasks ranging from image and speech recognition~\cite{krizhevsky2017imagenet, graves2013speech}, language understanding~\cite{brown2020language}, or game playing~\cite{silver2018general}. They have also been the key enabling technology in realizing perception modules for autonomous driving (see~\cite{grigorescu2020survey,mozaffari2020deep,chen2021deep,pandharipande2023sensing} for recent survey). Nevertheless, the brittleness of the DNN~\cite{szegedy2013intriguing} has been one of the safety concerns~\cite{abrecht2023deep} contributing to the potential hazards. 

Among many of the hardening techniques,  \emph{Randomized Smoothing (RS)}~\cite{cohen2019certified,salman2019provably, yang2020randomized} is one of the promising model-agnostic methods to ensure the robustness of the DNN in the operational design domain (ODD), where the underlying idea is to make predictions based on the aggregation of outputs from multiple variations of the input data perturbed with random noise. For certification, however, it is imperative to understand how much an RS-enabled DNN can withstand noise, often characterized using the concept of \emph{robustness radius}, i.e.,  the degree of perturbation where the majority-vote prediction remains consistent. This activity must be applied to the set of representative data points sampled within the ODD. The immediate consequence is its high computational cost: with~$m$ data points sampled from the ODD and each data point being perturbed~$n$ times, the number of DNN executions needed equals $mn$, which can be tremendous. It is also not a single-step process, as the introduction of continuous assurance~\cite{schleiss2022towards,bensalem2023continuous}, safety certificates should be produced whenever the DNN module is updated. Thus, it is crucial to have highly efficient algorithms in computing the point-wise robustness certificate. 

In this paper, we investigate the efficient computation of robustness certificates for each data point. Currently, the estimation of point-wise robustness certificates for RS-enabled DNNs requires creating \emph{hundreds of thousands} of samples (i.e., $n \approx 100000$), as demonstrated by prior results~\cite{cohen2019certified, salman2019provably, carlini2022certified}.  We counteract with existing know-how and demonstrate that computing such a point-wise robustness certificate for RS-enabled DNNs can be done with substantially fewer samples, with a small reduction in the size of the derived robustness radius while having the same confidence level.  
We provide the theoretical explanation of such a phenomenon based on the Central Limit Theorem (CLT) serving as an approximation to the binomial distribution (and thereby to the Clopper-Pearson interval calculation), integrating Shore's numerical approximation~\cite{shore1982simple} on the inverse cumulative function of the normal distribution. The dual form of the theoretical result can be used as an \emph{early stopping criterion} to know if adding additional samples can lead to a significantly increased radius matching the acceptance criterion.  Finally, the theoretical results are confirmed by extensive experiments on the standard CIFAR-10 and ImageNet datasets, where reducing the number of samples by two to three orders of magnitude still allows for providing a point-wise robustness certificate of adequate radius.

In summary, the main contribution of this paper includes the following:

\begin{itemize}
    \item An empirical evaluation of understanding the amount of perturbed inputs and its effect on the derived robustness radius for RS-enabled DNNs.

    \item A theoretical interpretation of the observed phenomenon and the mathematical formulation between the reduction/addition of samples and the decrease/increase of robustness radius. 
 
\end{itemize}

\section{Related Work}

Robustness certification is an important aspect of the safe commissioning of learning-enabled systems. For DNNs, knowing the robustness bound can be done by viewing the DNN as a mathematical object (i.e., a program without loops). This enables formal verification techniques such as SMT~\cite{katz2017reluplex}, mixed-integer programming~\cite{cheng2017maximum,tjeng2017evaluating},  abstract interpretation~\cite{gowal2018effectiveness, gehr2018ai2, singh2019abstract}, or convex relaxations~\cite{wong2018scaling}. While formal methods are computationally expensive, by viewing the DNN as a composition of functions, one can also apply the Lipchitz analysis~\cite{leino2021globally, singla2021improved}  to understand the impact of perturbation, thereby forming an estimate on the robustness bound. While these methods are directly applied to the DNN, they can also serve as a non-probabilistic lower bound for RS-enabled DNN, as the robustness bound computed by these methods ensures that all variations will share the same result, thereby creating an uncontested (i.e., predictions under all noises are the consistent) majority. 

Randomized smoothing~\cite{cohen2019certified} currently represents the state-of-the-art sampling-based robustness prediction methods, as it scales to large DNNs used in practice and is agnostic to their architectural details. Moreover, RS has been extended to additional threat models beyond standard $L_2$ balls, such as general $L_p$ norms~\cite{yang2020randomized}, geometric transformations~\cite{fischer2020certified}, segmentation~\cite{fischer2021scalable} and more. Nevertheless, the amount of sample variations remains a practical concern. To this end, the closest approach to our work is in~\cite{chen2022input}, where the authors studied the effect of using small sampling numbers (e.g., $n=1000$); if the certified radius with the small sample size is larger than the accepted radius with high confidence, then there is no need to use more samples. They also qualitatively studied the effect of radius change concerning the sample size decrease from a constant. On the contrary, we start by considering the ideal case of infinite sampling (i.e., $n = \infty$) and formally derive the performance decrease when the sample size is small (e.g., $n=1000$). Our proved theorems allow us to analytically estimate the potential increase in the point-wise robustness radius (Thm.~\ref{thm:radius_drop}) as well as the average radius increase in the whole domain (Thm.~\ref{thm:aver_radius_drop};  subject to a certain shape commonly observed in the experiment) due to the increase of samples. The duality of this result enables an early stopping scheme when the increase of sample size is deemed ineffective in achieving the radius; their method will continue to adaptively increase the sample size without an end.

Finally, randomized smoothing is one of the driving forces for ensuring the robustness of prediction, where apart from majority votes, uncertainty can also be estimated. Other sampling-based methods such as MC-dropout~\cite{gal2016dropout} also sampling-based methods by randomly disabling some parameters. It is also interesting to know the robustness radius of these methods, which has not been explored to the best of our knowledge.

\section{Preliminaries: Randomized Smoothing}

Let $f: \mathbb{R}^d \rightarrow \{1, \ldots, K\}$ be a classifier mapping inputs $\mathbf{x} \in \mathbb{R}^d$ into $K$ classes. In RS, $f$ is replaced with the following classifier:

\begin{equation} \label{eq:g}
    g_{\sigma}(\mathbf{x}) = \underset{y}{\arg\max} \;\mathbb{P}[f(\mathbf{x} + \mathbf{z}) = y], \mathbf{z} \sim N(0, \sigma^2 I)
\end{equation}

That is, $g_{\sigma}$ perturbs the input $\mathbf{x}$ with noise $\mathbf{z}$ that follows an isotropic Gaussian distribution $N(0, \sigma^2 I)$, and returns the class $A\in \{1, \ldots, K\}$ that gets the majority vote, i.e., the one that $f$ is most likely to return on the perturbed inputs.

Let $p_A(\mathbf{x}, \sigma) \defeq  \mathbb{P}[f(\mathbf{x} + \mathbf{z}) = A], \mathbf{z} \sim N(0, \sigma^2 I)$ be the probability of the majority class being~$A$, where we use the term $p_A$ when the context is clear. If $p_A \geq 0.5$, then $g_{\sigma}$ is guaranteed to be robust around $\mathbf{x}$, where we define the (guaranteed) \emph{robustness radius} in such a situation as follows:

\vspace{-5mm}

\begin{equation}\label{eq:R}
 R_{\sigma}  \defeq \begin{cases}
\sigma \Phi^{-1}(p_A) &\text{if $p_A \geq  0.5$}\\
0 &\text{otherwise}
\end{cases}
\end{equation}

where $\Phi^{-1}$ is the inverse of the normal cumulative distribution function (CDF). The robustness radius~$R_{\sigma}$ ensures consistency of prediction, namely
$\forall \mathbf{z} \in \mathbb{R}^d: ||\mathbf{z} ||_{2}\leq R_{\sigma} \rightarrow g_{\sigma}(\mathbf{x}) = g_{\sigma}(\mathbf{x}+\mathbf{z})$. The intuition is that a slight perturbation on $\mathbf{x}$ can change the output of $f$ arbitrarily, but not the one of~$g_{\sigma}$: since $g_{\sigma}$ relies on the consensus over points distributed around $\mathbf{x}$, a small shift cannot change a distribution much. This is the crucial fact where RS resides. 

Finally, notice that finding the precise value of $p_A$ is not possible; however, a lower bound $\bar{p_A}$ can be estimated by Monte Carlo sampling with a high degree of confidence $1 - \alpha$, as demonstrated in line~6 of Algo.~\ref{alg:certify} (technical details will be expanded in later sections). Yet, following existing results~\cite{cohen2019certified, salman2019provably, carlini2022certified}, the required number of samples~$n$ are typically around~$10000$ to~$100000$. 

\begin{algorithm}[tb]
   \caption{Finding the robustness radius (adapted from~\cite{cohen2019certified})}
   \label{alg:certify}
\begin{algorithmic}[1]
   \STATE {\bfseries Input:} point $\mathbf{x}$, classifier $f$, $\sigma$, $n$, $\alpha$
   \STATE {\bfseries Output:} class $A$ and robustness radius $R_{\sigma}$ of $\mathbf{x}$ 
   \STATE sample $n$ noisy samples $\mathbf{x}_1', ..., \mathbf{x}_n' \sim N(\mathbf{x}, \sigma^2 I)$
   \STATE get majority class $A \gets \arg\max_y \sum_{i=1}^n \mathbf{1}[f(\mathbf{x}_i') = y]$
   \STATE $\text{counts}(A) \leftarrow \sum_{i=1}^n \mathbf{1}[f(\mathbf{x}_i') = A]$
   \STATE $\bar{p_A} \leftarrow \text{LowerConfBound}(\text{counts}(A), n, \alpha)$ \COMMENT{compute probability lower bound}
   \IF{$\bar{p_A} \geq \frac{1}{2}$}
   \STATE return $A, \sigma \Phi^{-1}(\bar{p_A})$ 
   \ELSE
   \STATE return $\textsf{ABSTAIN}$
   \ENDIF
\end{algorithmic}
\end{algorithm}

\section{Reducing the Sample Size}\label{sec:method}

In this section, we perform the theoretical analysis\footnote{For important results, the created theorems use $\approx$ (approximately equal) to omit the error terms introduced by numerical approximation; it is an easy exercise to precise all terms, but the resulting formula would be too complex for the reader to grasp the big picture. } concerning the effect of reducing the sample size in Algo.~\ref{alg:certify} on the robustness radius and average certified accuracy. The critical point to explore is the dependence of the lower bound $\bar{p_A}$ and the number of samples~$n$ (as well as the admissible error rate~$\alpha$).

\subsection{The General Approach}
Let $p_A$ be probability of input~$\mathbf{x}$ being  class~$A$, when being fed into~$g_{\sigma}$ (Eq.~\eqref{eq:g}). As $p_A$ is unknown to us, the aim is to estimate it by drawing samples, with the goal of obtaining a lower bound $\bar{p_A}$ where the true probability is larger than $\bar{p_A}$ with confidence at least~$1 - \alpha$. This interprets the meaning of Line~$6$ in Algo.~\ref{alg:certify}. 

More specifically, let $\mathbf{x}_i' \sim N(\mathbf{x}, \sigma^2 I)$ be noisy versions of $\mathbf{x}$ ($i = 1,...,n$), and set $Y_i = \mathbf{1}[f(\mathbf{x}_i') = A]$; $Y_i$ is an indicator Random Variable (RV), taking the value $1$ if $f(\mathbf{x}_i')$ predicts the correct class~$A$, and~$0$ otherwise. Notice that $Y_i$'s are binomial RVs, with success probability $p_A$. Further, let $\hat{p} = \frac{Y_1 + ... + Y_n}{n}$ be the sample mean, i.e., the empirical estimate of $p_A$.

Given $\hat{p}$, $n$ and $\alpha$, there are many ways to estimate the lower bound as used line~$6$ of Algo.~\ref{alg:certify}. One standard approach is to apply the Clopper-Pearson test~\cite{clopper1934use} to obtain a lower bound, a term we call~$\bar{p_A}^{CP}$.
Unfortunately, while the Clopper-Pearson test gives us an exact lower bound for binomials, there is no analytical closed-form solution. Therefore, to obtain an algebraically tractable approximation, we apply the Central Limit Theorem (CLT) as our main tool~\cite{wasserman2004all}, 
which states that the distribution of sample means approximates a normal distribution as the sample size gets larger ($n \geq 30$), with mean $\E[\hat{p}] = p_A$ and variance $\Var[\hat{p}] = \frac{p_A (1 - p_A)}{n}$: 

\begin{equation} \label{eq:CLT}
    \hat{p} \sim N\left(p_A, \frac{p_A (1 - p_A)}{n}\right)  
\end{equation}

Before proving the lower-bound, we now establish a lemma characterizing the approximation of expectations over functions.

\begin{lemma}
\label{lemma:lemma_fEX}
Let $X$ be an RV with finite mean and variance, and let~$f$ be a twice continuously differentiable function, with $|f''(x)| \leq 2M$ for all $x \in \mathbb{R}$. Then the following condition holds.

\begin{equation}
    \label{eq:lemma_fEX}
    f(\E[X]) - M \cdot \mathrm{\emph{\textbf{Var}}}[X] \leq \E[f(X)] \leq f(\E[X]) + M \cdot \mathrm{\emph{\textbf{Var}}}[X] 
\end{equation}

\end{lemma}

\begin{proof}
    Since $f$ is twice continuously differentiable on the open interval with $f''$ continuous on the closed interval between $x_0$ and $x$, using Taylor's theorem with the Lagrange form of remainder,  one derives:

    \vspace{-3mm}
    \begin{equation}
        f(x) = f(x_0) + f'(x_0) (x - x_0) + \frac{1}{2} f''(\xi) (x - x_0)^2
    \end{equation}

    with $\xi \in (x_0, x)$. Since $|f''(x)| \leq 2M$ for all $x \in \mathbb{R}$, the above gives the following inequality:

    \vspace{-3mm}
    \begin{equation}\label{eq:inequality.1}
        \begin{split}
            f(x_0) + f'(x_0) (x - x_0) - M (x - x_0)^2 \leq f(x) \\
            \leq f(x_0) + f'(x_0) (x - x_0) + M (x - x_0)^2
        \end{split}
    \end{equation}

    Now, replace  $x$ by $X$ and $x_0 $ by $\E[X]$ in Eq.~\eqref{eq:inequality.1}, and take expectations on both sides, we get the following:

    \vspace{-3mm}
    \begin{equation}\label{eq:inequality.2}
        \begin{split}
            \E[f(\E[X])] + \E[f'(\E[X]) (X - \E[X])] - \E[M (X - \E[X])^2] \\ \leq \E(f(X)) \\
            \leq       \E[f(\E[X])] + \E[f'(\E[X]) (X - \E[X])] + \E[M (X - \E[X])^2] 
        \end{split}
    \end{equation}

Eq.~\eqref{eq:inequality.2} can be simplified to  Eq.~\eqref{eq:lemma_fEX}, due to the following facts. 
   \begin{itemize}
       \item $\E(f(\E[X])) = f(\E[X])$, 
       \item $\E[f'(\E[X])(X - \E[X])] = f'(\E[X]) \E[X - \E[X]] = \\ f'(\E[X]) (\E[X] - \E[X]) = 0$, and 
       \item $\E[M (X - \E[X])^2] = M \Var[X]$
   \end{itemize} 

       \vspace{-3mm}
\end{proof}

By applying the CLT via Eq.~\eqref{eq:CLT}, we derive a lower-bound for~$p_A$ as follows:

\begin{lemma} \label{thm:lower_bound_clt}
Let $Y_1, ..., Y_n$ be Bernoulli RVs, with success probability $p_A$ where $0 < p_l \leq p_A \leq p_h < 1$ with $p_l, p_h$ constants, and $\hat{p} = \frac{Y_1 + ... + Y_n}{n}$. Assume~$n \geq 30$ such that CLT holds. Then the following two conditions hold:
\begin{enumerate}

    \item $\bar{p_A}^{CP} \approx \hat{p} - z_{\alpha} \sqrt{ \frac{\hat{p} (1 - \hat{p})}{n} } $, where $z_{\alpha} = \Phi^{-1}(1 - \frac{\alpha}{2})$ is the $1 - \frac{\alpha}{2}$ quantile of the normal distribution~$N(0,1)$.

    \item $\E[\bar{p_A}^{CP}]$, i.e., the expected value of $\bar{p_A}^{CP}$, is equal to $p_A - z_{\alpha} \sqrt{ \frac{p_A (1 - p_A)}{n} } + \delta$, where $ \delta \in [-c \Var[\hat{p}], c \Var[\hat{p}]]$ with~$c$ being a constant.
\end{enumerate}
\end{lemma}

\begin{proof}
The first item is the standard normal interval approximation for the binomial, under the CLT approximation~\cite{brown2001interval}. For the second item, consider the function $f(p) = p - z_a \sqrt{\frac{p(1-p)}{n}}$. For $0 < p_l \leq p_A \leq p_h < 1$, $|f''(p)| =  \frac{z_a}{4 \sqrt{n} [p(1-p)]^{3/2} }$ is bounded by some constant $c$. 

By taking Lemma~1 where $X$ is assigned with~$\hat{p}$ and $M$ with $c$, we obtain

\begin{equation}\label{eq:moving.inside}
    f(\E[\hat{p}]) -  c\Var[\hat{p}] \leq \E[f(\hat{p})] \leq f(\E[\hat{p}]) + c\Var[\hat{p}]
\end{equation}

By applying condition 1,  using the definition of~$f$, and applying Eq.~\ref{eq:moving.inside}, we obtain the following. 
\begin{align*}
    \E[\bar{p_A}^{CP}] \approx \E[\hat{p} - z_{\alpha} \sqrt{ \frac{\hat{p} (1 - \hat{p})}{n} } ] = \E[f(\hat{p})]   \Rightarrow \\ 
    \E_{\hat{p}} [f(\hat{p})] \in [ f(\E[\hat{p}]) -  c\Var[\hat{p}],  f(\E[\hat{p}]) +  c\Var[\hat{p}]] 
\end{align*}

Finally, as $\E_{\hat{p}}[\hat{p}]$ equals $p_A$ (following the definition in Eq.~\ref{eq:CLT}), we can derive $\E_{\hat{p}}[\bar{p_A}^{CP}] \approx p_A - z_{\alpha} \sqrt{ \frac{p_A (1 - p_A)}{n} } + \delta$ where
$\delta \in [-c\Var[\hat{p}], c\Var[\hat{p}]]$, establishing the validity of the second condition.

\end{proof}

\paragraph{(Remark)} In Lemma~\ref{thm:lower_bound_clt}, the assumption on $\delta$ being negligible is reasonable in practice, e.g., $\delta \in [-0.0006, 0.0006]$ even for~$p_A = 0.95$, with $n = 1000$. 

Fig.~\ref{fig:clopper_vs_gauss} illustrates the lower-bound computed using Clopper-Pearson method and Lemma~\ref{thm:lower_bound_clt}. For each point in the line, it is created by fixing the~$n$ value (number of samples), repeating the trial for~$100$ times, followed by taking the average. We observe that the distance between them is small, and decreases rapidly as~$n$ exceeds~$100$.

\begin{figure}[t]
\centering
\vspace{-7mm}
\includegraphics[width=1.0\columnwidth]{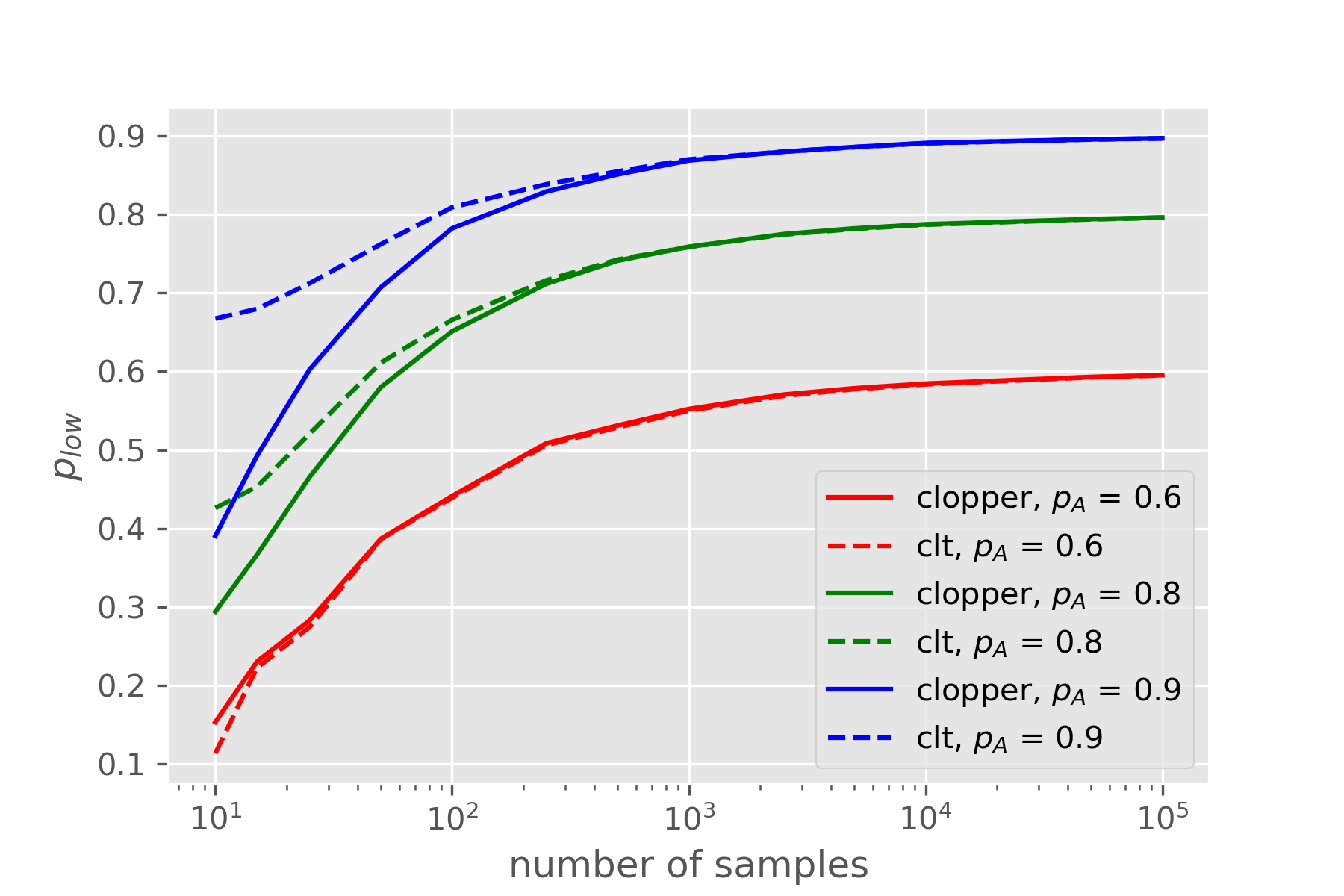}
\caption{Plot of the average lower bounds for the Clopper-Pearson method and Lemma~\ref{thm:lower_bound_clt}, obtained over~$100$ trials, for multiple values of~$p_A$ and~$n$.}
\label{fig:clopper_vs_gauss}
\end{figure}

\subsection{The Impact of Sample Size on Robustness Radius}

We now study how the sample size influences the robustness radius. Let $R_{\sigma}^{\alpha, n}(p_A)$ be the expected robustness radius estimated with~$n$ samples and error rate~$\alpha$ subject to noise level~$\sigma$, provided that the original probability is~$p_A$ but being approximated by $\bar{p_A}^{CP}$. That is,  $R_{\sigma}^{\alpha, n}(p_A) \defeq \E[\sigma \Phi^{-1}(\bar{p_A}^{CP})]$. Ideally, if we had an infinite number of samples, $\bar{p_A}^{CP}$ will be equal to~$p_A$, and the robustness radius would be, by Eq.~\eqref{eq:R}: $R_{\sigma} = R_{\sigma}^{0,\infty}(p_A) = \sigma \Phi^{-1}(p_A)$. 
However, for a finite number of samples, we do not have access to~$p_A$, but to a lower bound of it, as we saw before. 

To study how much the robustness radius, $R_{\sigma}^{\alpha, n}(p_A)$, drops due to the reduction of finite samples, we use the following approximation stated in Eq.~\eqref{eq:invCDF_approx} for~$\Phi^{-1}(p)$, which is valid for~$p \geq \frac{1}{2}$ \cite{shore1982simple}. In our case, the validity assumption for~$p \geq \frac{1}{2}$ is not a problem, as otherwise, if $p_A < \frac{1}{2}$, the robustness radius is $0$, implying that~$g_{\sigma}$ failed to predict the correct class (which is by common sense, undesired).

\begin{equation}\label{eq:invCDF_approx}
   \Phi^{-1}(p) \approx \frac{1}{0.1975} [p^{0.135} - (1 - p)^{0.135}]
\end{equation}

By applying Lemma~\ref{thm:lower_bound_clt} followed by integrating Eq.~\eqref{eq:invCDF_approx}, we can examine the influence $n$ has on $R_{\sigma}^{\alpha, n}(p_A)$, which leads to the following theorem.

\begin{theorem} \label{thm:radius_drop}
For an RS classifier~$g_{\sigma}$, given a point $\mathbf{x}$, assume that we estimate~$p_A$ by drawing~$n$ samples and compute the lower bound from the empirical $\hat{p}$ with confidence~$1-\alpha$. Assume that the conditions in Lemma~\ref{thm:lower_bound_clt} hold, while~ $|\frac{d^2 \Phi^{-1}(p) }{dp^2}| \Var[\hat{p}]$ and~$\delta$ in Lemma~\ref{thm:lower_bound_clt} are negligible. Then

\begin{equation}
    \label{eq:radius_drop_def}
    R_{\sigma}^{\alpha, n}(p_A) \approx \sigma \Phi^{-1} (p_A - t_{\alpha, n}) 
\end{equation}
where $t_{{\alpha}, n} = z_{{\alpha}} \sqrt{ \frac{p_A (1 - p_A)}{n} }$. In addition, $R_{\sigma}^{\alpha, n}(p_A)$ is approximately equal to:

\vspace{-3mm}
\begin{equation} \label{eq:radius_drop}
\begin{split}
    R_{\sigma}^{\alpha, n}(p_A) \approx 5.063 \sigma [ p_A^{0.135} - (1 - p_A)^{0.135} - \\
        0.135 \frac{z_{{\alpha}}}{\sqrt{n}} (p_A^{-0.365} (1 - p_A)^{1/2} + p_A^{1/2} (1 - p_A)^{-0.365}) ]
\end{split}
\end{equation}
\end{theorem}

The practical usage of Thm.~\ref{thm:radius_drop} occurs when we have already used $n$ samples to compute the radius. With~$\sigma$ being a constant, under a given specified~$\alpha$, one can reverse-engineer to uncover the appropriate $p_A$ value making Eq.~\ref{eq:radius_drop} holds. This allows us to predict the robustness radius increase when one changes from $n$ to substantially larger values such as~$100n$. The maximum increase occurs when $n \rightarrow \infty$, where the robustness radius is approximately equal to the term in Eq.~\eqref{eq:r.a.infty}.

\vspace{-2mm}
\begin{equation}\label{eq:r.a.infty}
 R_{\sigma}^{0, \infty}(p_A)  \approx   5.063 \;\sigma [ p_A^{0.135} - (1 - p_A)^{0.135}]
\end{equation} 

\begin{proof}

As the condition of Lemma~\ref{thm:lower_bound_clt} holds, $\bar{p_A}^{CP} \approx \hat{p} - t_{\alpha, n}$. Using Eq.~\eqref{eq:lemma_fEX}, we get 
\begin{equation}
\begin{split}
\sigma \Phi^{-1}(\E[\bar{p_A}^{CP}]) -  M \Var[\hat{p}] \\\leq R_{\sigma}^{\alpha, n}(p_A) = \E[\sigma \Phi^{-1}(\bar{p_A}^{CP})] \\
\sigma \Phi^{-1}(\E[\bar{p_A}^{CP}]) +  M \Var[\hat{p}]
\end{split}
\end{equation}

where $M$ is the upper bound of $|\frac{d^2 \Phi^{-1}(p) }{dp^2}|$ in the interval $[p_l, p_h)$. As $|\frac{d^2 \Phi^{-1}(p) }{dp^2}|\Var[\hat{p}]$ is  negligible, we have: 
\begin{equation}
\begin{split}
    R_{\sigma}^{\alpha, n}(p_A) = \E[\sigma \Phi^{-1}(\bar{p_A}^{CP})] \approx \sigma \Phi^{-1}(\E[\bar{p_A}^{CP}]) 
\end{split}
\end{equation}
By applying the second condition of Thm~\ref{thm:lower_bound_clt} and with $\delta$ being negligible, we have the following.

\begin{equation}
\begin{split}
    R_{\sigma}^{\alpha, n}(p_A) \approx \sigma \Phi^{-1}(\mathbf{E}_{\hat{p}}[\bar{p_A}^{CP}]) \approx \sigma \Phi^{-1}(p_A - t_{\alpha,n}) 
\end{split}
\end{equation}

Next, we replace $\Phi^{-1}$ by the approximation of Eq.~\eqref{eq:invCDF_approx}, and we get:

\begin{equation} \label{eq:cert_radius_approx}
\begin{split}
    R_{\sigma}^{\alpha, n}(p_A) \approx \sigma \frac{1}{0.1975} [(p_A - t_{\alpha, n})^{0.135} - (1 - p_A + t_{\alpha, n})^{0.135}]
\end{split}
\end{equation}

For further simplification, we use binomial theorem, $(1 + x)^a = 1 + ax + \frac{a(a-1)}{2!} x^2 + ...$ valid for $|x| < 1$ on both terms of Eq.~\eqref{eq:cert_radius_approx}, and keep only the 1st order terms. Doing that gives:

\begin{equation}\label{eq:binom_approx}
\begin{split}
    & A \defeq \left( p_0 - z_{\alpha} \sqrt{ \frac{p_A (1 - p_A)}{n} } \right)^{0.135} \\
    & = p_A^{0.135} \left( 1 - \frac{z_{\alpha}}{\sqrt{n}} p_A^{-1/2} (1 - p_A)^{1/2} \right)^{0.135} \Rightarrow \\ 
    & A \approx p_A^{0.135} (1 - 0.135 \frac{z_{\alpha}}{\sqrt{n}} p_A^{-1/2} (1 - p_A)^{1/2}) = 
                p_A^{0.135} \\
    &- 0.135 \frac{z_{\alpha}}{\sqrt{n}} p_A^{-0.365} (1 - p_A)^{1/2} \\
    & B \defeq \left(1 - p_A + z_{\alpha} \sqrt{ \frac{p_A (1 - p_A)}{n}} \right)^{0.135} = \\
    & (1 - p_A)^{0.135}  \left( 1 + \frac{z_{\alpha}}{\sqrt{n}} p_A^{1/2} (1 - p_A)^{-1/2} \right)^{0.135} \Rightarrow \\  
    & B \approx (1 - p_A)^{0.135} (1 + 0.135 \frac{z_{\alpha}}{\sqrt{n}} p_A^{1/2} (1 - p_A)^{-1/2}) \\
    & = (1 - p_A)^{0.135} + 0.135 \frac{z_{\alpha}}{\sqrt{n}} p_A^{1/2} (1 - p_A)^{-0.365}
\end{split}
\end{equation}

Substituting in Eq.~\eqref{eq:cert_radius_approx} and combining terms gets Eq.~\eqref{eq:radius_drop}.

\end{proof}

\paragraph{(Remark)} In Thm.~\ref{thm:radius_drop}, the assumption on $|\frac{d^2 \Phi^{-1}(p) }{dp^2}| \Var[\hat{p}]$ being negligible is reasonable, as $\Var[\hat{p}] \defeq \frac{p_A (1 - p_A)}{n}$, and when~$n$ is around~$1000$, the value can at most be~$0.00025$.
The second derivative of inverse normal CDF
$|\frac{d^2 \Phi^{-1}(p) }{dp^2}|$, when~$p$ is not too close to~$1$, is reasonably sized. For example, when $p=0.9$, $|\frac{d^2 \Phi^{-1}(p) }{dp^2}|= 27.77$, making the product term $|\frac{d^2 \Phi^{-1}(p) }{dp^2}| \Var[\hat{p}] = 0.0069$ still small. 
We observe in the experiments that even when~$n$ is not very big (cf. Sec.~\ref{sec:experiments}), the approximation and the observed behavior remain similar.

\subsection{Average Robustness Radius Drop}
In the previous section, we examine the effect of $n$ on the robustness radius for one specific point. But ultimately, what really interests us is the effect on the average robustness radius over the entire dataset. This is the average robustness radius that we expect to lose by reducing the number of samples. 

For any particular point, the robustness radius depends on $p_A$, the probability that $g_{\sigma}$ outputs the correct class $A$ (Eq.~\eqref{eq:R}). To answer questions about the average behavior of the robustness radius, we need to consider the probability distribution of $p_A$: we devote the probability density function (pdf) of $p_A$ to be~$\Pr(p_A)$. We can roughly imagine $\Pr(p_A)$ as a histogram over the $p_A$ values obtained from our dataset. 

Formally, the average robustness radius is then given by Eq.~\eqref{eq:def.average.radius}. The integration can start at $0.5$, as the robustness radius $R_{\sigma}^{\alpha, n}(p_A)$ equals $0$ for $p_A < 0.5$ (cf. Eq.~\eqref{eq:R}).

\vspace{-5mm}
\begin{equation}\label{eq:def.average.radius}
\begin{split}
    \bar{R}_{\sigma}(\alpha, n) \defeq \E_{\Pr(p_A)} [R_{\sigma}^{\alpha, n}(p_A)] \\
    = \int_{0}^{1} R_{\sigma}^{\alpha, n}(p_A) \Pr(p_A) dp_A  
    = \int_{0.5}^{1} R_{\sigma}^{\alpha, n}(p_A) \Pr(p_A) dp_A  
    \end{split}
\end{equation}

\begin{figure}[t]
\centering
\includegraphics[width=1.0\columnwidth]{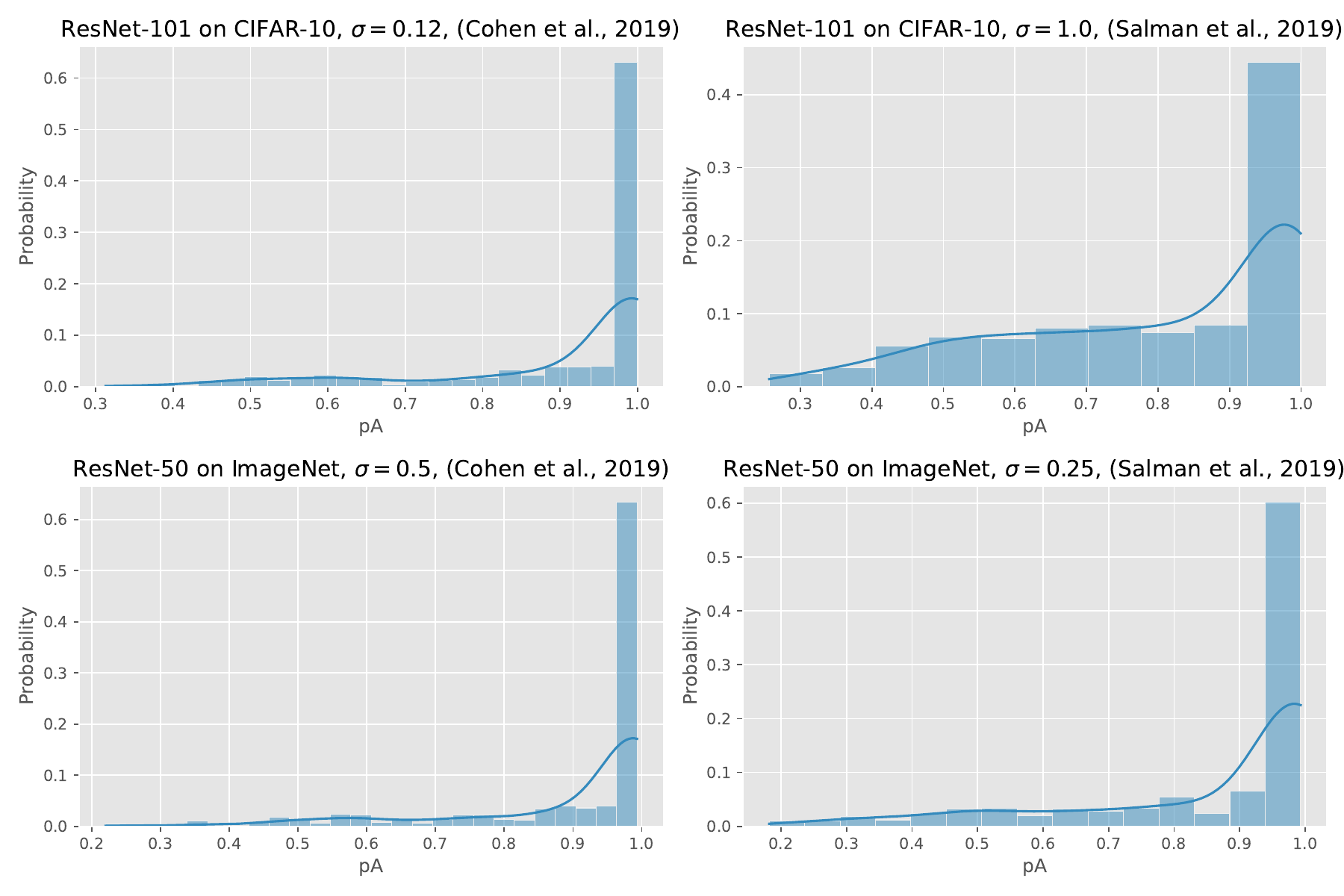}
\caption{Plots of histograms and density plots of $p_0$ obtained for different models and datasets, as shown in the figure titles. The values of $p_0$ were estimated empirically using $n = 100000$ samples.}. \label{fig:pA_distributions}

\vspace{-3mm}

\end{figure}

Unfortunately, $\Pr(p_A)$ depends heavily on the particular model and dataset used, and does not seem to follow any well-known family of distributions such as Gaussian. This can be seen in Fig.~\ref{fig:pA_distributions}, where we estimate the histogram of $p_A$ for different models of~\cite{cohen2019certified} and~\cite{salman2019provably}. Nevertheless, we notice that $\Pr(p_A)$ is skewed towards~$1$ in all cases we tested: namely, most of the mass of $\Pr(p_A)$ is concentrated in a small interval $(\beta, 1)$ on the right, while the mass outside it - and especially in the interval $[0, 0.5]$ is close to zero. Intuitively, this is the behavior we expect, where most of the data points in the sample are far from the decision boundary; otherwise, the robustness radius of a point would be very small.

Inspired by the shape of the distribution as illustrated in Fig.~\ref{fig:pA_distributions}, we form a distribution as characterized in Lemma~\ref{thm:compute}, where~$\kappa_1$ and~$\kappa_2$ shall be small, similar to that of Fig.~\ref{fig:pA_distributions}. 

\begin{lemma} \label{thm:compute}
Assume that $\Pr(p_A)$ follows a piecewise uniform distribution across input points $\mathbf{x}$, where 

\vspace{-5mm}
\begin{equation}\label{eq:img.generation}
\Pr(p_A) \defeq \begin{cases}
\kappa_1 &\text{if $p_A \in [0, 0.5)$}\\
\kappa_2 &\text{if $p_A \in [0.5, \beta)$}\\
\kappa_3 \defeq \frac{1-  \kappa_1 (0.5) -\kappa_2(\beta - 0.5)}{1-\beta} &\text{if $p_A \in [\beta, 1)$}
\end{cases}
\end{equation}

Then $\bar{R}_{\sigma}(\alpha, n) = \kappa_2 \int_{0.5}^{\beta} R_{\sigma}^{\alpha, n}(p_A)dp_A + 
  \kappa_3 \int_{\beta}^{1} R_{\sigma}^{\alpha, n}(p_A) dp_A $

\end{lemma}

\begin{proof}
    The proof immediately follows Eq.~\eqref{eq:def.average.radius} and~\eqref{eq:img.generation}.
\end{proof}

By integrating Eq.~\eqref{eq:cert_radius_approx} into Lemma~\ref{thm:compute}, with~$\kappa_1$,~$\kappa_2$ and~$\beta$ being constants, the integrals can be calculated. 
In the following, we exemplify the computation by computing the result of robustness radius drop when $\kappa_1 = 0$ and $\kappa_2 = 0$, a case where $\Pr(p_A)$ is uniform in the interval $[\beta, 1)$ with cases $\beta \geq 0.8$ and $\beta = 0.5$.

\begin{theorem} \label{thm:aver_radius_drop}
Assume that $\Pr(p_A)$ follows a uniform distribution in the interval $[\beta, 1)$ across data points $\mathbf{x}$ in the input domain
The decrease of the average certified radius $\bar{R}_{\sigma}(\alpha, n)$ using $n$ samples from the ideal case of $n = \infty$ is approximately equal to:

\begin{equation}\label{eq:aver_radius_drop}
    r_{\sigma}(\alpha, n) \coloneqq \frac{ \bar{R}_{\sigma}(\alpha, n) }{ \bar{R}_{\sigma}(0, \infty) } \approx 1 - \Theta \frac{z_{\alpha}}{\sqrt{n}} 
\end{equation}

where 
\begin{equation}\label{eq:theta.value}
\Theta \defeq \begin{cases}
1.64 &\text{if $\beta \in [0.8, 1)$}\\
2 &\text{if $\beta = 0.5$}
\end{cases}
\end{equation}

\end{theorem}

\begin{proof} We proceed with the proof by separating cases.

\noindent \textbf{[Case 1 - $\beta \in [0.8, 1)$]}  Recall that Eq.~\eqref{eq:radius_drop} gives us $R_{\sigma}^{\alpha, n}(p_A)$ for a particular point with class probability $p_A$, while Eq.~\eqref{eq:r.a.infty} provides the result of   $R_{\sigma}^{0,\infty}(p_A)$. Consider the ratio characterized below. 

\begin{equation}\label{eq:r.r.ratio}
\begin{split}
    \frac{R_{\sigma}^{\alpha, n}(p_A)}{R_{\sigma}^{0,\infty}(p_A)} = 1 - 0.135 \frac{z_{{\alpha}}}{\sqrt{n}} h(p_A)
\end{split}
\end{equation}

where
\begin{equation}
\begin{split}
    h(p_A) = \frac{p_A^{-0.365} (1 - p_A)^{1/2} + p_A^{1/2} (1 - p_A)^{-0.365}}{p_A^{0.135} - (1 - p_A)^{0.135}}
\end{split}
\end{equation}

 Crucially, $h(p_A)$ is almost constant within an interval close to $1$, as illustrated in Fig.~\ref{fig:r_quotient}. For instance, in the interval $(\beta, 1)$ with $\beta \geq 0.8$, we find $h(p_A) \approx 12.14$. Substituting this value inside Eq.~\eqref{eq:r.r.ratio}, we obtain:

\begin{figure}[t]
\centering
\vspace{-6mm}
\includegraphics[width=0.9\columnwidth]{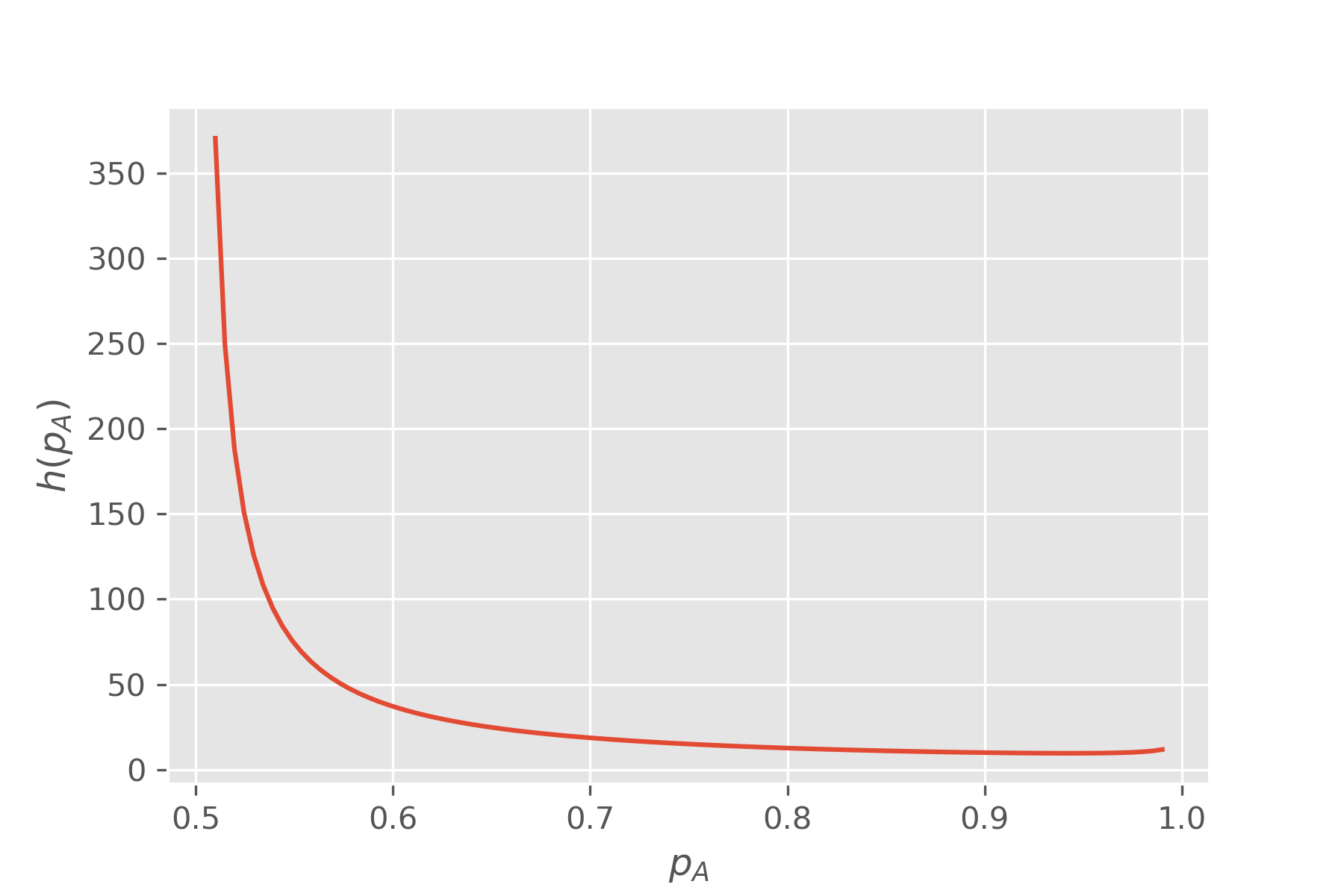}
\caption{Plot of $h(p_A)$ in the interval $[0.5, 1]$}
\label{fig:r_quotient}
\end{figure}

\begin{equation}
    \frac{R_{\sigma}^{\alpha, n}(p_A)}{R_{\sigma}^{0,\infty}(p_A)} \approx 1 - 1.64 \frac{z_{{\alpha}}}{\sqrt{n}}
\end{equation}

Therefore:
\vspace{-5mm}
\begin{equation}\label{eq:proof}
\begin{split}
    \bar{R}_{\sigma}(\alpha, n) = \int_{0}^{1} R_{\sigma}^{\alpha, n}(p_A) \Pr(p_A) dp_A \\
    \approx (1 - 1.64 \frac{z_{{\alpha}}}{\sqrt{n}}) \int_{\beta}^{1} R_{\sigma}^{0,\infty}(p_A) \Pr(p_A) dp_A  \\
    =   (1 - 1.64 \frac{z_{{\alpha}}}{\sqrt{n}}) \int_{0}^{1} R_{\sigma}^{0,\infty}(p_A) \Pr(p_A) dp_A  \\
=
 (1 - 1.64 \frac{z_{{\alpha}}}{\sqrt{n}}) \bar{R}_{\sigma}(0, \infty)
\end{split}
\end{equation}

In Eq.~\ref{eq:proof}, the equality of expanding the integral from~$\int_{\beta}^{1} $ to~$\int_{0}^{1}$ comes from the fact that $\Pr(p_A)=0$ when $p_A\in [0, \beta)$. As $\int_{\beta}^{1} R_{\sigma}^{0,\infty}(p_A) \Pr(p_A) dp_A$ is exactly the definition of $\bar{R}_{\sigma}(0, \infty)$, we obtain the required formula. Interestingly, the derivation in this case holds for density functions $Pr[p_A]$ in $[\beta, 1)$ of any form.

\vspace{2mm}
\noindent \textbf{[Case 2 - $\beta = 0.5$]} 
When $\beta = 0.5$, $\Pr(p_A) = 2$ when $p_A \in [0, 1)$. For the average robustness radius, we get:

\vspace{-2mm}
\begin{equation}
\begin{split}
    \bar{R}_{\sigma}(\alpha, n)  
    = \int_{0}^{1} R_{\sigma}^{\alpha, n}(p_A) \Pr(p_A) dp_A = 2 \int_{0.5}^{1} R_{\sigma}^{\alpha, n}(p_A)dp_A 
  \end{split}
\end{equation}

Substituting Eq.~\eqref{eq:radius_drop}, we can perform the integration and obtain:

\vspace{-3mm}
\begin{equation}\label{eq.r.intermediate}
\begin{split}
    \bar{R}_{\sigma}(\alpha, n) = 2 \int_{0.5}^{1} 5.063 \sigma [ p_A^{0.135} - (1 - p_A)^{0.135} \\
    - 0.135 \frac{z_{\alpha}}{\sqrt{n}} (p_A^{-0.365} (1 - p_A)^{1/2} + p_A^{1/2} (1 - p_A)^{-0.365}) ] dp_A 
\end{split}
\end{equation}

The integrals of the form $p_A^a$ and $(1 - p_A)^a$ can be computed easily, while the integrals of the terms $p_A^a (1 - p_A)^b$ are integrals of the Beta function, and can be evaluated numerically. Based on the calculations, we get: 

\vspace{-3mm}
\begin{equation}
\begin{split}
    & \bar{R}_{\sigma}(\alpha, n) = \sigma \left( 0.796 - 1.603 \frac{z_{\alpha}}{\sqrt{n}} \right)
\end{split}
\end{equation}

Finally, diving the terms, we see that the ratio of robustness radius drop is independent of $\sigma$, and is approximately equal to:

\begin{equation}
    \frac{ \bar{R}_{\sigma}(\alpha, n) }{ \bar{R}_{\sigma}(0, \infty) } \approx 1 - 2 \frac{z_{\alpha}}{\sqrt{n}} 
\end{equation}
which concludes the proof.

\end{proof}

\vspace{-5mm}

\paragraph{(Observation)} (i) Note that even when one changes the distribution and falls back to the form in Eq.~\ref{eq:img.generation}, the computation in Eq.~\ref{eq.r.intermediate} will be changed, but in the final computation of $\frac{ \bar{R}_{\sigma}(\alpha, n) }{ \bar{R}_{\sigma}^{\infty} }$, $\sigma$ appearing on both the denominator and the numerator will be canceled out, implying that the average robustness radius drop is \emph{independent of the noise level}~$\sigma$. (ii) The two cases we have demonstrated hint at a tendency where~$\Theta$ shall fall inside the interval~$[1.64, 2]$, when the distribution adjusts from uniform to a small peak.

\subsection{Certified Accuracy Drop}
Apart from the average robustness radius, another important quantity in evaluating robust classifiers is the average certified accuracy: We denote it by $acc_{R}$, which is the fraction of points that~$g_{\sigma}$ classifies correctly, and with robustness radius at least larger than a threshold~$R$.

To understand the situation, consider again the distribution of $\Pr(p_A)$, and assume that we are evaluating $acc_{R_0}$ for some radius of interest $R_0$. By Eq.~\eqref{eq:R}, this corresponds to a probability $p_0$:

\begin{equation}
    R_0 = \sigma \Phi^{-1}(p_0) \Leftrightarrow \\
    p_0 = \Phi(R_0 / \sigma)
\end{equation}

By applying similar simplifying assumptions on~$\Pr(p_A)$ as in Thm.~\ref{thm:aver_radius_drop}, we can obtain an upper bound on the certified accuracy drop. One can easily extend the result to accommodate the distribution characterized in Eq.~\eqref{eq:img.generation}.

\begin{theorem} \label{thm:cert_acc_drop}
Let $acc_{R_0}(\alpha, n)$ be the certified accuracy $g_{\sigma}$ obtains using~$n$ samples and error rate~$\alpha$, and let $acc_{R_0}$ be the ideal case where $n = \infty$. Assume that $\Pr(p_A)$ follows a uniform distribution in the interval $[0.5, 1)$ across input points $\mathbf{x}$. The certified accuracy drop, $\Delta acc_{R_0}(\alpha, n) = acc_{R_0} - acc_{R_0}(\alpha, n)$ satisfies:

\vspace{-2mm}
\begin{equation}\label{eq:cert_acc_drop}
    \Delta acc_{R_0}(\alpha, n) \leq \frac{z_{\alpha}}{\sqrt{n}} 
\end{equation}
\end{theorem}

\begin{proof}
Let $p_0 = \Phi(R_0 / \sigma)$; then, for $acc_{R_0}$ we have that:

\begin{equation}
    acc_{R_0} = \int_{p_0}^{1} \Pr(p_A) dp_A
\end{equation}

Nevertheless, when we use~$n$ samples, we can measure only the $(1-\alpha)$-lower bound of~$p_A$, which, by Theorem~\ref{thm:lower_bound_clt}, is approximately equal to: $\bar{p_A}^{CP} = p_A - t_{\alpha, n}$. 

So, now a point will be included in the integration if we have~$\bar{p_A}^{CP} \geq p_0$. Via syntactic rewriting, we have

\vspace{-5mm}
\begin{equation}
\begin{split}
    \bar{p_A}^{CP} \geq p_0 \Rightarrow 
    p_A - t_{\alpha, n} \geq p_0 \Rightarrow 
    p_A \geq p_0 + t_{\alpha, n}
\end{split}
\end{equation}

For $t_{\alpha, n}$ we notice that:
\vspace{-2mm}
\begin{equation}
\begin{split}
    t_{\alpha, n} = z_{\alpha} \sqrt{ \frac{p_A (1 - p_A)}{n} } \Rightarrow  
    t_{\alpha, n} \leq \frac{z_{\alpha}}{2 \sqrt{n}}
\end{split}
\end{equation}
since the quantity $p_A (1 - p_A)$ with $p_A \in [0, 1]$ is maximized for $p_A = 0.5$, and has value $1/4$.

Hence, all points satisfying $p_A \geq p_0 + \frac{z_{\alpha}}{2 \sqrt{n}}$ will be included in the integration, and the interval that will be excluded will be at most $[p_0, p_0 + \frac{z_{\alpha}}{2 \sqrt{n}}]$. So, we finally obtain:
\vspace{-2mm}
\begin{equation}
\begin{split}
    \Delta acc_{R_0}({\alpha}, n) \leq \int_{p_0}^{1} \Pr(p_A) dp_A - \int_{p_0 + \frac{z_{\alpha}}{2 \sqrt{n}}}^{1} \Pr(p_A) dp_A \Rightarrow \\
    \Delta acc_{R_0}(\alpha, n) \leq \int_{p_0}^{p_0 + \frac{z_{\alpha}}{2 \sqrt{n}}} \Pr(p_A) dp_A
\end{split}
\end{equation}

Under the assumption that $\Pr(p_A)$ is uniform in $[0.5, 1)$, we have $\Pr(p_A) = 2$, and the last integral is simply $\frac{z_{\alpha}}{2 \sqrt{n}} \times 2$, which yields the required result. 
\end{proof}

\section{Evaluation}\label{sec:experiments}

\begin{table}[t]
  \centering
  \caption{Average robustness radius for each noise level $\sigma$ and sample size $n$ on CIFAR-10, for the models of \cite{cohen2019certified} (with $\alpha = 0.001$)}
  \label{tab:cohen_cifar}
  
  \begin{tabular}{| c || c | c | c | c | c | c | c |}
    \hline
    $\sigma / n$  & 25 & 50 & 100 & 250 & 500 & 1000 & 10000 \\
    \hline \hline
    0.12 & 0.055 & 0.091 & 0.121 & 0.154 & 0.174 & 0.192 & 0.238  \\
    \hline
    0.25 & 0.09 & 0.152 & 0.203 & 0.258 & 0.292 & 0.319 &  0.385  \\
    \hline
    0.50 & 0.119 & 0.206 & 0.276 & 0.354 & 0.395 & 0.43 & 0.501  \\
    \hline
    1.00 & 0.111 & 0.202 & 0.284 & 0.371 & 0.41 & 0.448 & 0.513  \\
    \hline
\end{tabular}
\end{table}

\begin{table}[t]
  \centering
  \caption{Average robustness radius for each noise level $\sigma$ and sample size $n$ on ImageNet, for the models of \cite{cohen2019certified} (with $\alpha = 0.001$)}
  \label{tab:cohen_imagenet}
  
  \begin{tabular}{| c || c | c | c | c | c | c | c |}
    \hline
    $\sigma / n$  & 25 & 50 & 100 & 250 & 500 & 1000 & 10000 \\
    \hline \hline
    0.25 & 0.095 & 0.155 & 0.208 & 0.266 & 0.301 & 0.333 & 0.477  \\
    \hline
    0.50 & 0.148 & 0.241 & 0.326 & 0.414 & 0.471 & 0.602 & 0.734  \\
    \hline
    1.00 & 0.189 & 0.313 & 0.425 & 0.537 & 0.603 & 0.663 & 0.875  \\
    \hline
\end{tabular}
\end{table}

In this section, we analyze experimentally the influence of $n$ on the average robustness radius and certified accuracy, and compare the theoretical results of Sec.~\ref{sec:method} with the actual measurements by running Algo.~\ref{alg:certify}. We work with the standard CIFAR-10~\cite{krizhevsky2009learning} and ImageNet~\cite{deng2009imagenet} datasets, as done in several seminal papers on RS~\cite{cohen2019certified, salman2019provably, carlini2022certified}. 

To measure the influence of~$n$ (sample size), we take the classifiers from Cohen et al.~\cite{cohen2019certified} made available for repeatability evaluation. They have trained different models that work best with the corresponding~$\sigma$. We fix $\alpha = 0.001$ which is the typical value, and measure the average robustness radius as a function of~$n$, for CIFAR-10 and ImageNet. 
We subsample every $20$-th example in CIFAR-10 and every $100$-th in ImageNet, following the protocol of~\cite{cohen2019certified}. The results can be seen in Table~\ref{tab:cohen_cifar} and~\ref{tab:cohen_imagenet}. To inspect these also visually, we plot the measured dependency of $\bar{R}_{\sigma}(\alpha, n)$ from $n$ for the different $\sigma$ values, along with the predictions of Eq.~\eqref{eq:aver_radius_drop}. For each experiment, we approximate the ratio $r_{\sigma}(\alpha, n)$ using as $\bar{R}_{\sigma}^{\infty}$ the value we obtain for $n = 100000$ samples. The results are shown in Fig.~\ref{fig:cohen_cifar_formula} and Fig.~\ref{fig:cohen_imagenet_formula}.

\begin{figure*}[htp]
    \centering
    \begin{subfigure}{0.32\textwidth}
        \includegraphics[width=\linewidth]{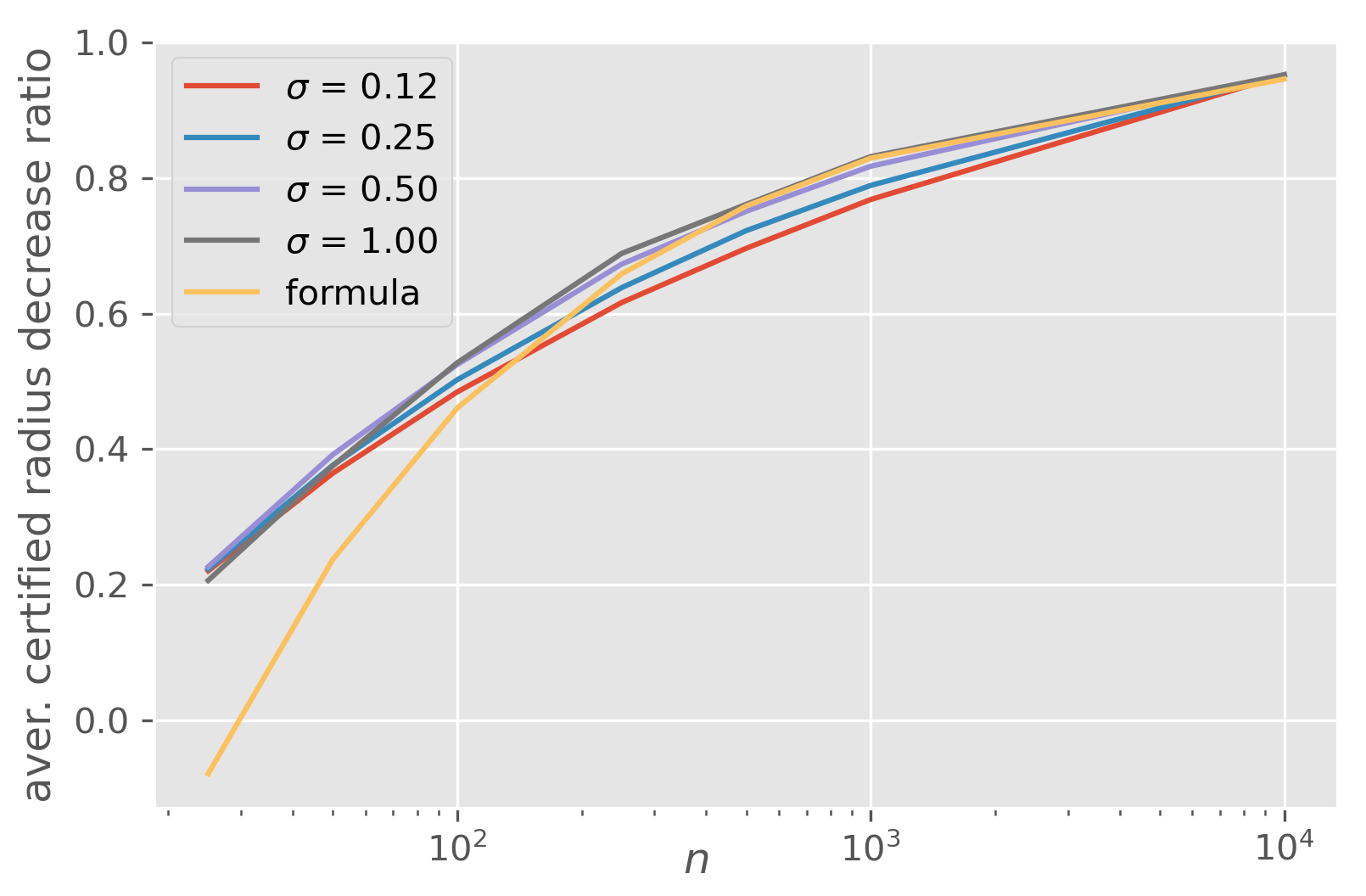}
\caption{Average robustness radius reduction for each noise level $\sigma$ and sample size $n$ on CIFAR-10, for the models of \cite{cohen2019certified} (with $\alpha = 0.001$), along with the predictions of Eq.~\eqref{eq:aver_radius_drop}}\label{fig:cohen_cifar_formula}
    \end{subfigure}
    \hfill 
    \begin{subfigure}{0.32\textwidth}
        \includegraphics[width=\linewidth]{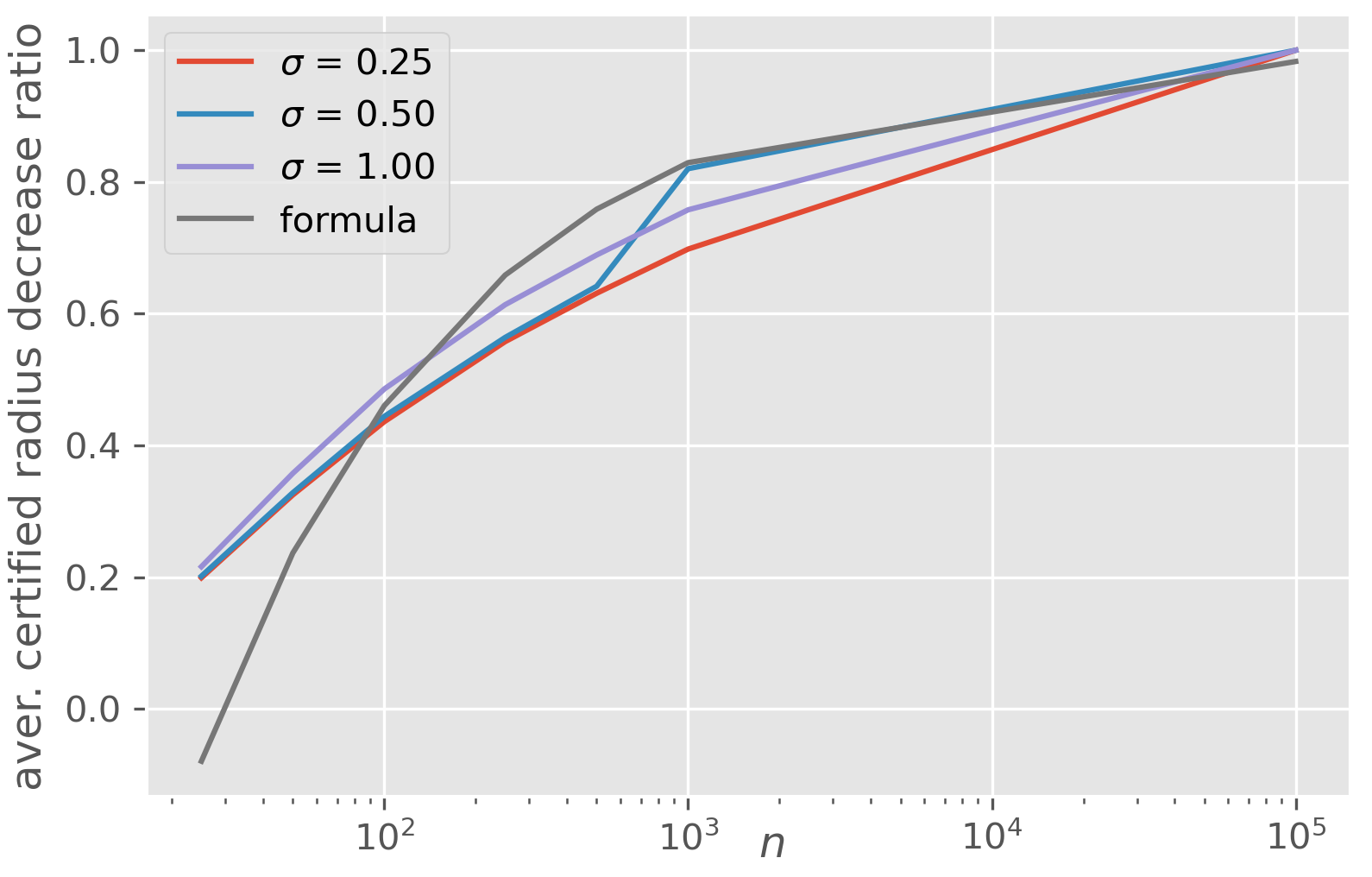}
\caption{Average robustness radius reduction for each noise level $\sigma$ and sample size $n$ on ImageNet, for the models of \cite{cohen2019certified} (with $\alpha = 0.001$), along with the predictions of Eq.~\eqref{eq:aver_radius_drop}}\label{fig:cohen_imagenet_formula}
    \end{subfigure}
    \hfill 
    \begin{subfigure}{0.32\textwidth}
        \includegraphics[width=\linewidth]{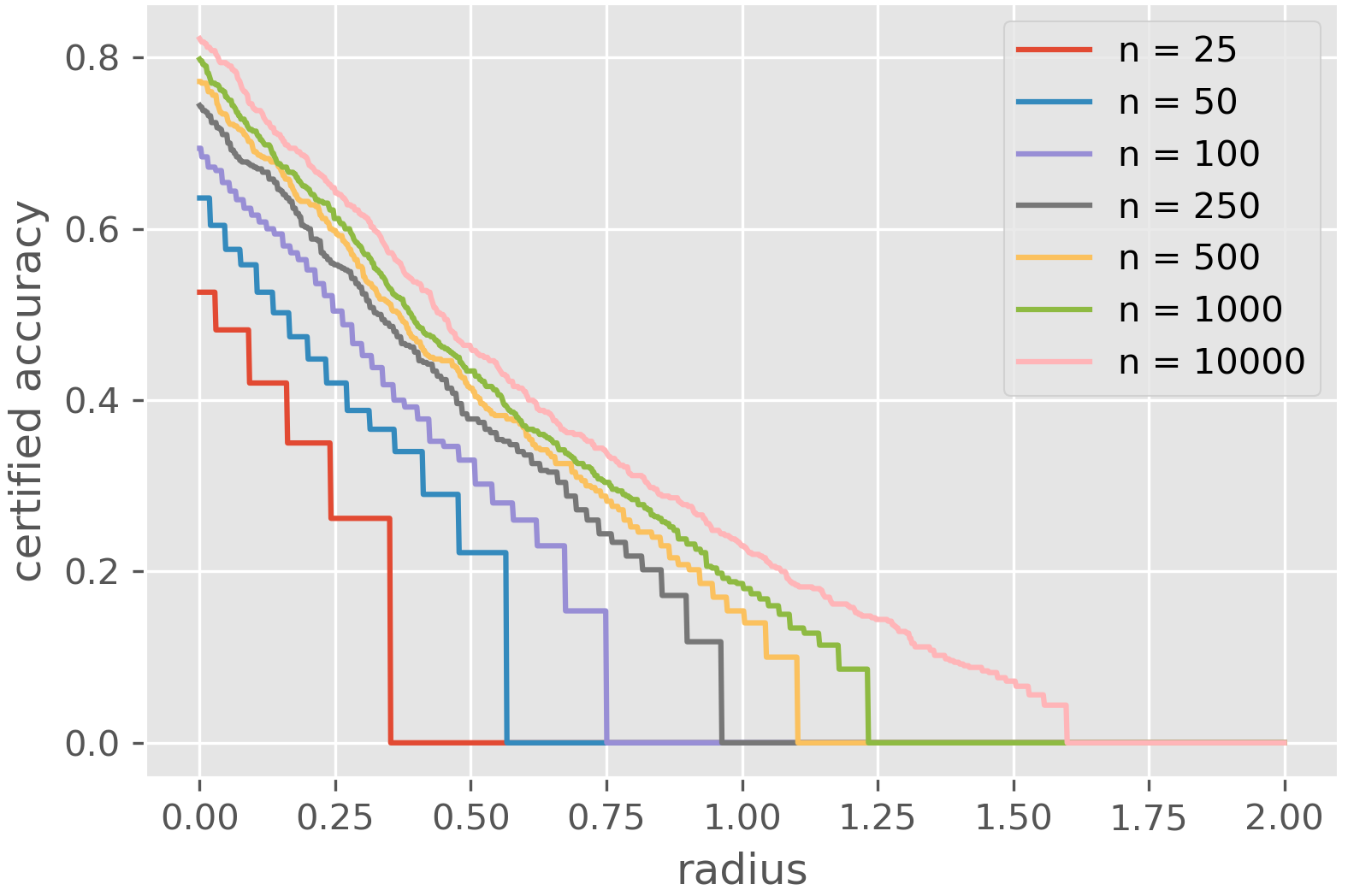}
\caption{Certified accuracy at $\sigma = 0.5$ as a function of $n$ on CIFAR-10, for the models of \cite{cohen2019certified} (with $\alpha = 0.001$)\vspace{2mm}}
\label{fig:cert_acc_cifar_cohen_sigma_050}
    \end{subfigure}

    \begin{subfigure}{0.32\textwidth}
        \includegraphics[width=\linewidth]{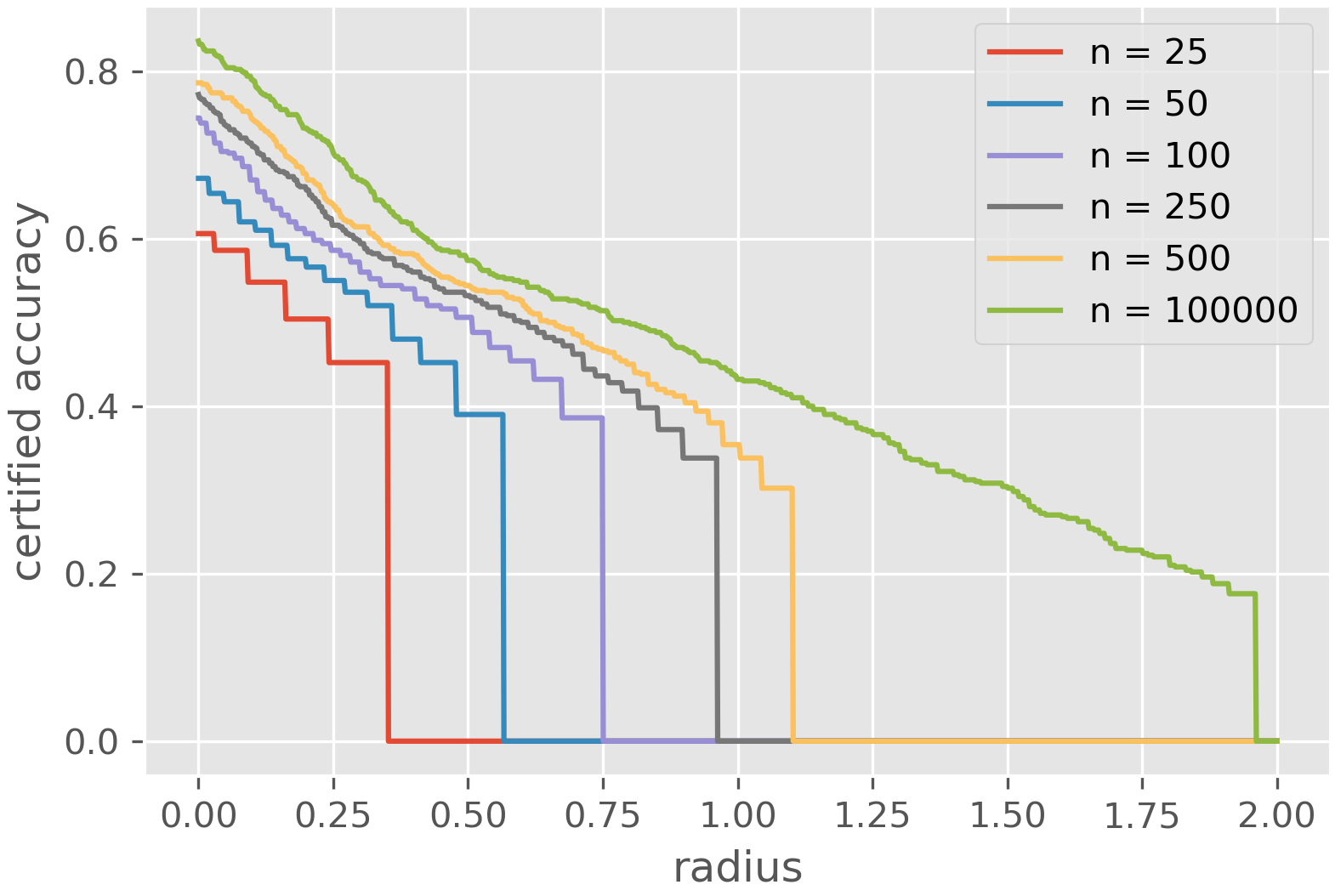}
\caption{Certified accuracy at $\sigma = 0.5$ as a function of $n$ on ImageNet, for the models of \cite{cohen2019certified} (with $\alpha = 0.001$) }
\label{fig:cert_acc_imagenet_cohen_sigma_050}
    \end{subfigure}
    \hfill 
    \begin{subfigure}{0.32\textwidth}
        \includegraphics[width=\linewidth]{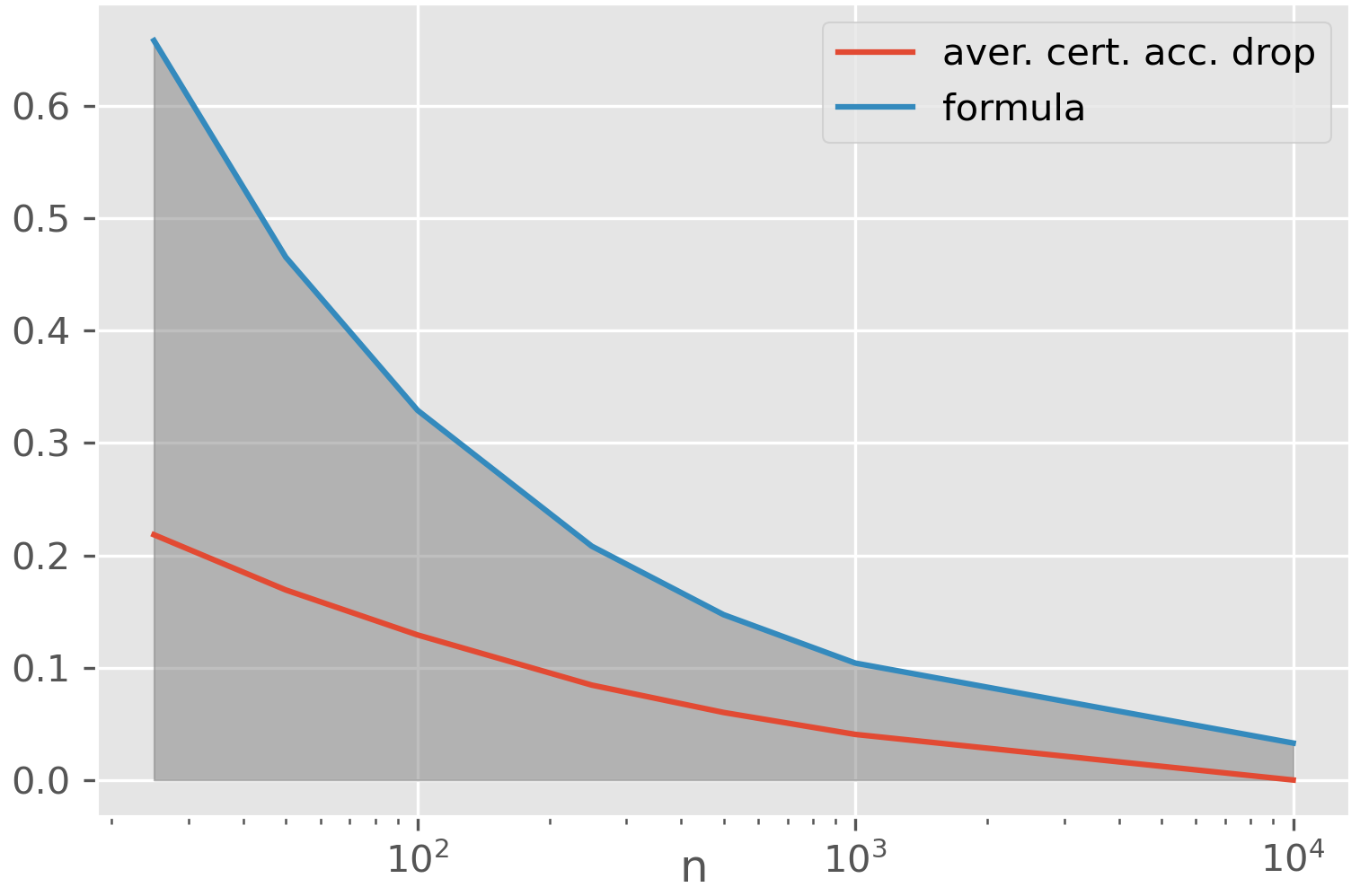}
\caption{Plot of average certified accuracy drop for the models of~\cite{cohen2019certified}, at $\sigma = 0.5$, along with the predictions of Eq.~\eqref{eq:cert_acc_drop} (CIFAR-10).}
\label{fig:cert_acc_cifar_cohen_sigma_050_vs_formula}
    \end{subfigure}
    \hfill 
    \begin{subfigure}{0.32\textwidth}
        \includegraphics[width=\linewidth]{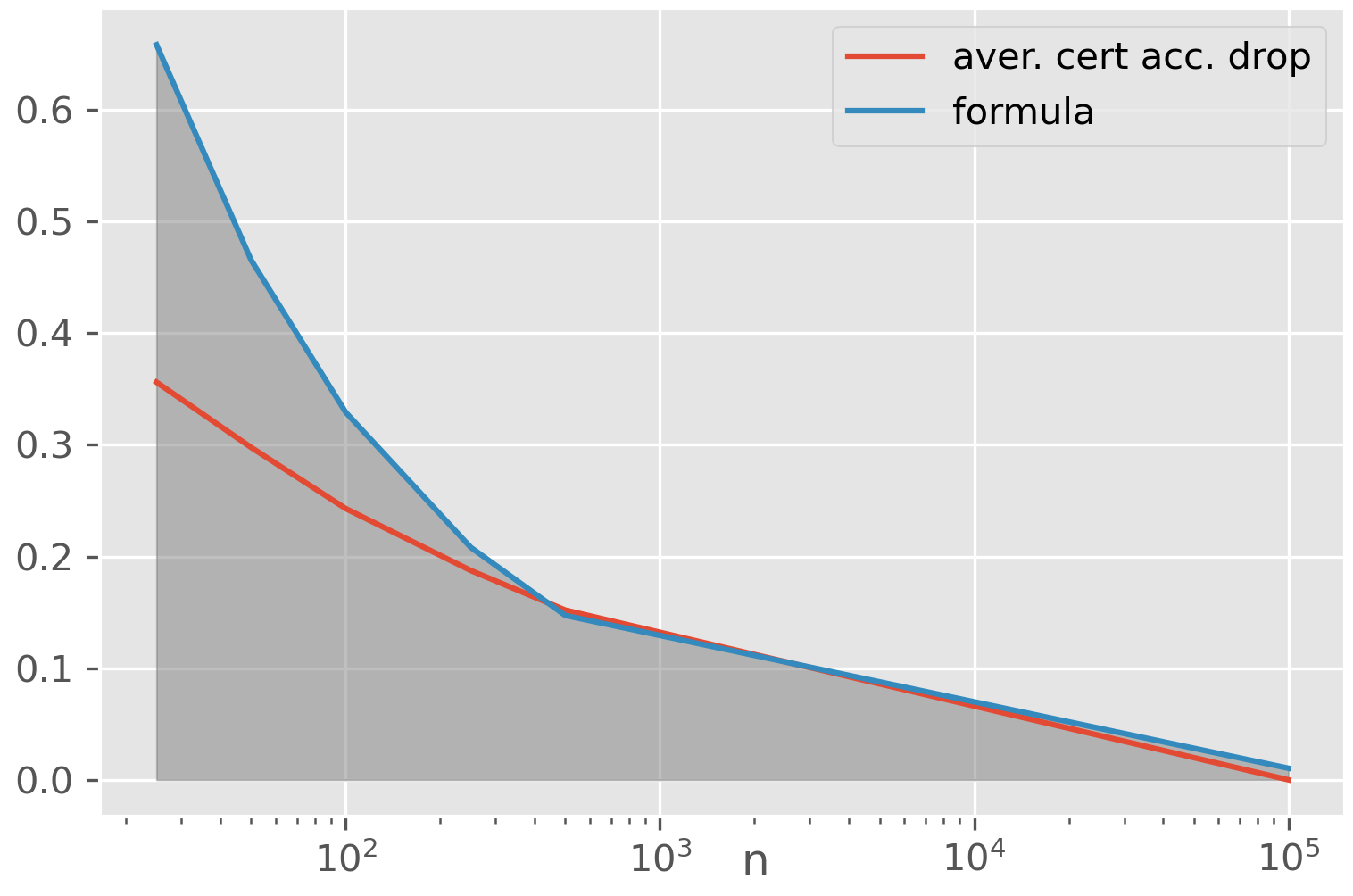}
\caption{Plot of average certified accuracy drop for the models of~\cite{cohen2019certified}, at $\sigma = 0.5$, along with the predictions of Eq.~\eqref{eq:cert_acc_drop} (ImageNet).}
\label{fig:cert_acc_imagenet_cohen_sigma_050}
    \end{subfigure}

    \caption{Evaluation results}
\end{figure*}

From Table~\ref{tab:cohen_cifar} and~\ref{tab:cohen_imagenet}, we observe that the reduced sample sizes do not decrease the average robustness radius $\bar{R}_{\sigma}(\alpha, n)$ as much as expected: for example, in the case of CIFAR-10, a $10\times$ decrement (from $10000$ to $1000$) reduces $\bar{R}$ by only around $20\%$ across noise levels $\sigma$. Moreover, a $100\times$ decrement reduces $\bar{R}$ by only $50\%$. Similarly, for the case of ImageNet, a reduction of $n$ from $100000$ to $100$ reduces $\bar{R}$ by merely $50\%$.

Analyzing the measurements across all datasets and models, the results well-support the theoretical bound characterized by Thm.~\ref{thm:aver_radius_drop}. First, the radius drop is independent of the noise level~$\sigma$; indeed, in the experiments we found approximately the same radius reduction across different~$\sigma$ values for each dataset. Second, we observe that the reduction of $\bar{R}_{\sigma}(\alpha, n)$ from $n = 10000$ to $n = 1000$ is around $\approx 85\%$, which is what we see in the experiments. Similarly, the formula shows that there is little difference for $n = 10000$ and $n = 100000$ also in agreement with the observations. On the other hand, the predicted reduction as we shift $n$ from $10000$ to $100$ is around $48\%$, which is slightly larger than the one we saw in experiments. This is to be expected, as Eq.~\eqref{eq:aver_radius_drop} captures the tendency of having a burst in the distribution (cf. Fig.~\ref{fig:pA_distributions}), and is "unaware" of the specific model and dataset details; recall that for every dataset and every value of $\sigma$, there is a corresponding distinct classifier provided by~\cite{cohen2019certified}. Thus, Eq.~\eqref{eq:aver_radius_drop} delivers decent predictions among~$2$ datasets across~$7$ different models.

Next, we want to investigate the effect of $n$ on certified accuracy. For that, we measure the certified accuracy of the models of~\cite{cohen2019certified} on CIFAR-10 and ImageNet: We choose one model and plot the certified accuracy curve for each value of~$n$. The results are shown in Fig.~\ref{fig:cert_acc_cifar_cohen_sigma_050} and Fig.~\ref{fig:cert_acc_imagenet_cohen_sigma_050},which show the certified accuracy of a model for each given radius $R_0$, $acc_{R_0}(\alpha, n)$.

We observe that the distance between the curves (e.g., the robustness radius drop) is roughly constant, until a curve drops to zero, in accordance with Eq.~\eqref{eq:cert_acc_drop}. Further, in order to compare the predictions of Eq.~\eqref{eq:cert_acc_drop} with reality, we plot the average certified accuracy drop, averaged over radii, and compare it to the theoretically expected value. This is illustrated in Fig.~\ref{fig:cert_acc_cifar_cohen_sigma_050_vs_formula} and~\ref{fig:cert_acc_imagenet_cohen_sigma_050}. The predictions of Eq.~\eqref{eq:cert_acc_drop} form a "conservative envelope" in Fig.~\ref{fig:cert_acc_cifar_cohen_sigma_050_vs_formula} and~\ref{fig:cert_acc_imagenet_cohen_sigma_050}; they are larger than the certified accuracy drop observed in practice. Notice that there is no strong guarantee that this must be so, as Thm.~\ref{thm:cert_acc_drop} relies on some simplifying assumptions that may not hold universally for all models and datasets; however, what we are mostly interested in is predicting the general trend, which Eq.~\eqref{eq:cert_acc_drop} seems to capture adequately.

\section{Conclusion}

In this paper, we addressed the challenge of the sample requirements for robustness certification based on randomized smoothing (RS). Our investigation revealed that significant reductions in the number of samples have a less severe impact on the average robustness radius and certified accuracy than previously anticipated. Through detailed empirical and theoretical analysis across multiple model architectures trained on CIFAR-10 and ImageNet, we observed consistent behavior.

Looking ahead, we see promising opportunities for applying our RS speedup techniques in AI safety. One potential application is in the robustness profiling of classifiers. Typically, such assessments would require testing with 100000 samples per input data point; however, our methodology could significantly lower the cost by enabling robustness assessments with as few as 100 samples, subsequently extrapolating to larger datasets. Moreover, the emergence of foundational models and their availability via API calls presents another area ripe for application. These models, whether Large Language Models (LLMs) or Vision-Language Models (VLMs), are susceptible to adversarial attacks. Our approach could feasibly enable robustness verification for each query without the need for~$100000$ samples, a previously untenable requirement.


\end{document}